\documentclass[journal]{IEEEtran}

\usepackage{amsfonts}
\usepackage[cmex10]{amsmath}
\usepackage{amssymb}
\usepackage[dvips]{graphicx}
\usepackage{pstricks}

\newtheorem{hproblem}{\bf Problem}
\newtheorem{hproposition}{\bf Proposition}
\newtheorem{hlemma}{\bf Lemma}
\newtheorem{hremark}{\it Remark}
\newtheorem{hproof}{\it Proof}

\newenvironment{proposition}{\begin{hproposition}\it}{\end{hproposition}}

\newenvironment{remark}{\begin{hremark}}{\hfill$\triangleleft$\end{hremark}}
\newenvironment{proof}{\begin{hproof}\it}{\end{hproof}}

 \title{\LARGE \bf
Bipedal locomotion using variable stiffness actuation%
\thanks{This work has been partly funded by the European Commission's Seventh
        Framework Programme as part of the project VIACTORS under grant
        no. 231554.}
}

\author{Ludo C. Visser, Stefano Stramigioli, and Raffaella Carloni%
\thanks{\{l.c.visser, s.stramigioli, r.carloni\}@utwente.nl,
        MIRA Institute, Faculty of Electrical Engineering, Mathematics and Computer Science,       
        University of Twente, The Netherlands.}
}

\begin{document}

\maketitle
\thispagestyle{empty}
\pagestyle{empty}

%%%%%%%%%%%%%%%%%%%%%%%%%%%%%%%%%%%%%%%%%%%%%%%%%%%%%%%%%%%%%%%%%%%%%%%%%%%%%%%%
\begin{abstract}
Robust and energy-efficient bipedal locomotion in robotics is still a challenging topic. In order to address issues in this field, we can take inspiration from nature, by studying human locomotion. The Spring-Loaded Inverted Pendulum (SLIP) model has shown to be a good model for this purpose. However, the human musculoskeletal system enables us to actively modulate leg stiffness, for example when walking in rough terrain with irregular and unexpected height variations of the walking surface. This ability of varying leg stiffness is not considered in conventional SLIP-based models, and therefore this paper explores the potential role of active leg stiffness variation in bipedal locomotion. It is shown that the conceptual SLIP model can be iteratively extended to more closely resemble a realistic (i.e., non-ideal) walker, and that feedback control strategies can be designed that reproduce the SLIP behavior in these extended models. We show that these extended models realize a cost of transport comparable to human walking, which indicates that active leg stiffness variation plays an important role in human locomotion that was previously not captured by the SLIP model. The results of this study show that active leg stiffness adaptation is a promising approach for realizing more energy-efficient and robust bipedal walking robots.
\end{abstract}

%%%%%%%%%%%%%%%%%%%%%%%%%%%%%%%%%%%%%%%%%%%%%%%%%%%%%%%%%%%%%%%%%%%%%%%%%%%%%%%%
\section{Introduction}
\label{sec:introduction}

Robust and energy-efficient bipedal locomotion in robotics is an interesting research topic with many open questions. In particular, on one side of the spectrum, robust bipedal robots are being developed, but without much consideration for energy efficiency \cite{BostonDynamics2011}. On the other side of the spectrum, extremely efficient bipedal locomotion is being achieved by exploiting passive robot dynamics \cite{Collins2005}. However, the gaits of such passive dynamic walkers lack robustness \cite{Schwab2001}.

In contrast, human walking is both robust and energy efficient, and, therefore, a better understanding of human walking could aid the design of robots achieving similar performance levels. To develop this understanding, models have been proposed that capture the essential properties of human gaits. One remarkably simple model is the bipedal spring-loaded inverted pendulum (SLIP) model, proposed in \cite{Geyer2006}. Despite its simplicity, the model accurately reproduces the hip trajectory and ground reaction force profiles observed in human gaits. Furthermore, the model encodes a wide variety of gaits, ranging from slow walking to running \cite{Rummel2010b}. 

It has been shown that the bipedal SLIP model can be used to generate reference gaits for a fully actuated bipedal robot \cite{Garofalo2012,Ketelaar2013}. However, the human musculoskeletal system enables the leg stiffness to be varied continuously, in order to adapt to different gaits and terrains. These stiffness variations also play a role in disturbance rejection, for example in uneven terrain with sudden and unexpected changes in the walking surface. Variable stiffness in the legs of the bipedal SLIP walker has been shown to increase energy efficiency \cite{vanderLinde1998,Ghorbani2008} and improve robustness \cite{Jafarian2011}, but continuous leg stiffness adaptation has not yet been considered. Recent advances in the field of variable stiffness actuators, a class of actuators that allow the actuator output stiffness to be changed independently from the actuator output position \cite{Vanderborght2013}, are enabling the realization of robotic walkers with physically variable leg stiffness. Therefore, further research into bipedal walking with variable leg stiffness should be pursued, with the aim of getting closer to realizing bipedal walking robots with human-like performance characteristics.

In this work, we explore modeling and control of bipedal walking with controllable leg stiffness. The aim of this work is to show that SLIP-like walking behavior can be embedded in more sophisticated models of bipedal walkers. This is achieved by the development of control strategies based on the principles of leg stiffness variation, inspired by the capabilities of the human musculoskeletal system. Starting from the conventional bipedal SLIP model, we iteratively extend this model, first by making the leg stiffness controllable, then by adding a swing leg and its dynamics, and then further by including a knee in the swing leg. In parallel, we derive a control strategy, which is extended with each model iteration. For each iteration, we show the stabilizing properties of the controller, demonstrating that the controller derived for the SLIP model in the first iteration is sufficiently robust to handle the increasingly complex dynamics in subsequent iterations. The final model and controller can serve as a template for bipedal robot control strategies.

The paper is organized as follows. In Section~\ref{sec:slip}, we revisit the bipedal SLIP model, as presented in \cite{Geyer2006}, and analyze its dynamics. Then, in Section~\ref{sec:vslip}, we extend the bipedal SLIP model to have controllable stiffness (the V-SLIP model, for Variable SLIP), and derive a control strategy that renders a natural gait of the SLIP model asymptotically stable. In Section~\ref{sec:swing}, the controlled V-SLIP model is extended to include feet, with the aim of introducing swing leg dynamics. The V-SLIP control strategy is extended to handle the extra degrees of freedom, and the stabilizing properties of the controller are demonstrated. The swing leg model is further refined in Section~\ref{sec:knee} by adding knees, with again further extension of the control strategy. A comparison of the models and their controllers is presented in Section~\ref{sec:simulation}, and Section~\ref{sec:conclusions} concludes the paper with a discussion and final remarks.

\subsection*{Conventions in Notation}

To avoid notational clutter, variable names will be reused for different models. However, this reuse is consistent, and variables with the same name indicate the same quantity in the various models. For example, $q_i$ denotes configuration variables, and $p_i$ denotes momentum variables, state vectors are named $x$, and $f(x)$ and $g_i(x)$ are drift and input vector fields on the state manifold respectively. The Lie-derivative of a function $h$ along a vector field $X$ is denoted by $\mathcal{L}_X h$. Function arguments are omitted where this is considered possible.

%%%%%%%%%%%%%%%%%%%%%%%%%%%%%%%%%%%%%%%%%%%%%%%%%%%%%%%%%%%%%%%%%%%%%%%%%%%%%%%%
\section{The Bipedal SLIP Model}
\label{sec:slip}

In this Section, we revisit the bipedal SLIP model, as presented in \cite{Geyer2006}. The model is depicted in Figure~\ref{fig:slip_model}. It consists of a point mass $m_h$, located in the joint connecting the two legs, i.e. the hip joint. The legs consist of massless telescopic springs of stiffness $k_0$ and rest length $L_0$. The configuration variables $(q_1,q_2) =: q$ describe the planar position of the point mass, with $(p_1,p_2) =: p$ the associated momentum variables. In the following, we derive the dynamic equations for this system, and analyze its dynamics.

\begin{figure}
\centering
\begin{pspicture}(7,5.4)
%\psgrid
 
 \rput[bl]{0}(0.2,0.2){%
    \includegraphics[width=0.75\columnwidth]{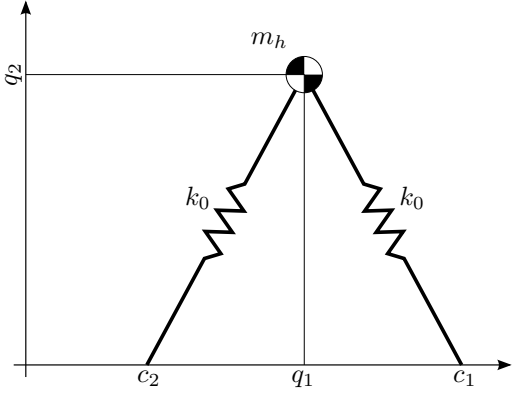}
 }
 
 \rput{0}(4.05,0.2){$q_1$}
 \rput{90}(0.2,4.25){$q_2$}
 \rput{0}(3.6,4.7){$m_h$}
 \rput{0}(2.65,2.6){$k_0$}
 \rput{0}(5.5,2.6){$k_0$}
 \rput{0}(2,0.2){$c_2$}
 \rput{0}(6.2,0.2){$c_1$}
\end{pspicture}
\caption{The bipedal SLIP model---The model consists of a point mass $m_h$, located in the hip joint, i.e. where two massless telescopic springs, with a constant spring stiffness $k_0$ and rest length $L_0$, are connected. The configuration variables $(q_1,q_2)$ describe the position of the hip.}
\label{fig:slip_model}
\end{figure}

\subsection{System Dynamics}
\label{subsec:slip_dynamics}

The bipedal SLIP model shows, for appropriately chosen initial conditions \cite{Geyer2006,Rummel2010b}, a passive walking gait as illustrated in Figure~\ref{fig:slip_gait}. In order to derive the dynamic equations that describe the gait of this model, two phases need to be considered: 1) two legs are in contact with the ground (i.e. the double support phase), and 2) one leg is in contact with the ground (i.e. the single support phase). Furthermore, we consider the parameter $\alpha_0$, which is the angle at which the massless leg touches down at the end of the single support phase, as indicated in Figure~\ref{fig:slip_gait}.

\begin{figure}
\centering
\begin{pspicture}(8.4,4.8)
 %\psgrid
 
 \rput[bl]{0}(0.2,0.2){%
  \includegraphics[width=3.15in]{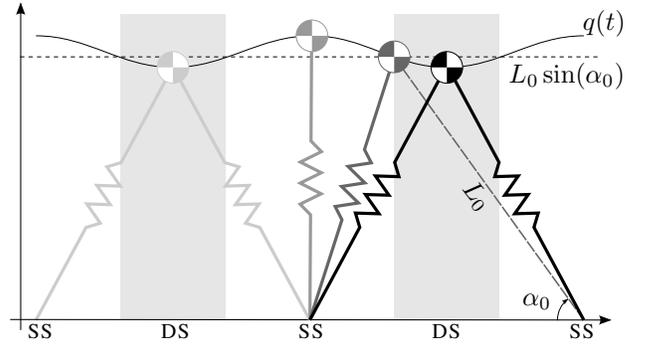}
 }
 
 \rput{-52}(6.4,2){$L_0$}
 \rput{0}(7.2,0.6){$\alpha_0$}
 
 \rput{0}(8.1,4.3){$q(t)$}
 \rput{0}(7.6,3.6){$L_0 \sin (\alpha_0)$}
 
 \rput{0}(0.6,0.2){\textsc{ss}}
 \rput{0}(2.4,0.2){\textsc{ds}}
 \rput{0}(4.2,0.2){\textsc{ss}}
 \rput{0}(6.0,0.2){\textsc{ds}}
 \rput{0}(7.8,0.2){\textsc{ss}}
\end{pspicture}
\caption{Passive gait of the bipedal SLIP model---The model alternates between single support (\textsc{ss}) and double support (\textsc{ds}) phases, depending on the hip height and the model parameters $L_0$ and $\alpha_0$. The gray shading will be used throughout this paper to indicate that the walker is in the double support phase.}
\label{fig:slip_gait}
\end{figure}

The contact conditions are determined by the spring rest length $L_0$ and angle of attack $\alpha_0$, as shown in Figure~\ref{fig:slip_gait}. In particular, if the system is in the single support phase, the touchdown event of the trailing leg occurs when
\begin{equation}\label{eq:vslip_td}
q_2 = L_0 \sin (\alpha_0),
\end{equation}
and at this moment the foot contact point $c_2$ is calculated as
\begin{equation*}
c_2 = q_1 + L_0 \cos (\alpha_0).
\end{equation*}
Conversely, when the system is in the double support phase, the transition to the single support phase occurs when either of the two springs reaches its rest length with non-zero speed, and thus loses contact with the ground, i.e. when
\begin{equation}\label{eq:vslip_lo}
\sqrt{(q_1 - c_i)^2 + q_2^2} = L_0, \quad i=1,2.
\end{equation}
In nominal conditions, only the trailing leg is allowed to lift off, after which the contact point $c_2$ is relabeled as $c_1$ to correspond to the notation used for the single support phase.

In order to derive the dynamic equations, we define the kinetic energy function $K = \frac{1}{2}p^T M^{-1} p$, where 
\begin{equation}\label{eq:vslip_mm}
M = \mathrm{diag}(m_h,m_h)
\end{equation}
is the mass matrix and $p := M \dot{q}$ are the momentum variables. The potential energy function is defined as
\begin{equation*}
V = m_h g_0 q_2 + \frac{1}{2}k_0 (L_0 - L_1)^2 + \frac{1}{2}k_0 (L_0 - L_2)^2,
\end{equation*}
where $L_i := \sqrt{(q_1 - c_i)^2 + q_2^2}$, and $g_0$ is the gravitational acceleration.  During the single support phase, we set $L_2 \equiv L_0$ to eliminate the influence of this virtually swinging leg.

The dynamic equations in Hamiltonian form are defined through the Hamiltonian energy function $H = K + V$ and given by
\begin{equation}\label{eq:slip_dynamics}
\frac{\text{d}}{\text{d} t} \left[ \begin{array}{c} q \\ p \end{array} \right] =
\left[ \begin{array}{cc} 0 & I \\ -I & 0 \end{array} \right]
\left[ \begin{array}{c} \frac{\partial H}{\partial q} \smallskip \\
                        \frac{\partial H}{\partial p} \end{array} \right].
\end{equation}

From \eqref{eq:slip_dynamics}, it can be noted that the configuration variables $q(t)$ are of class $C^2$. This is due to the fact that the $\frac{\partial V}{\partial q}$ is not differentiable at the moment of phase transition. This is because the massless second leg does not have a zero rate of change of length at the moment of touchdown, i.e.
\begin{equation*}
\frac{\text{d}}{\text{d} t} L_2 \Big|_{t = t_\mathrm{touchdown}^+} \neq 0,
\end{equation*}
where $t_\mathrm{touchdown}^+$ indicates that the time-derivative is taken on the right of the discontinuity. It will be  shown later that this has consequences for the controller design.

%%%%%%%%%%%%%%%%%%%%%%%%%%%%%%%%%%%%%%%%%%%%%%%%%%%%%%%%%%%%%%%%%%%%%%%%%%%%%%%%
\section{The Controlled V-SLIP Model}
\label{sec:vslip}

The passive bipedal SLIP model provides no control inputs, and therefore the only way to influence its behavior is by the choice of initial conditions. Therefore,  it is proposed to extend the bipedal SLIP model to have massless telescopic springs with \textit{variable} stiffness \cite{Visser2012a}. This bipedal V-SLIP (for Variable SLIP) model is depicted in Figure~\ref{fig:vslip_model}. The difference with respect to the bipedal SLIP model is that the leg stiffness now has a controllable part, i.e. $k_i = k_0 + u_i,i=1,2$. In this Section we give the dynamic equations for this system and present a stabilizing controller.

\begin{figure}
\centering
\begin{pspicture}(7,5.4)
 %\psgrid
 
 \rput[bl]{0}(0.2,0.2){%
  \includegraphics[width=2.625in]{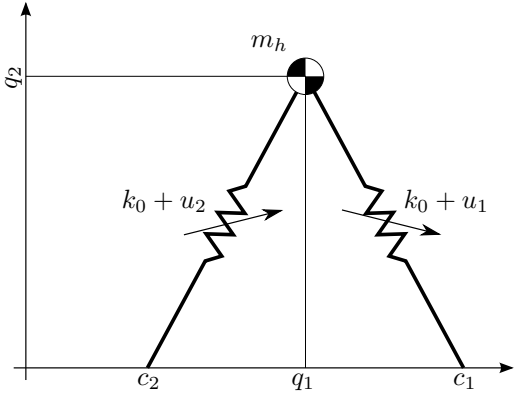}
 }
 
 \rput{0}(4.05,0.2){$q_1$}
 \rput{90}(0.2,4.25){$q_2$}
 \rput{0}(3.6,4.7){$m_h$}
 \rput{0}(2.2,2.6){$k_0 + u_2$}
 \rput{0}(5.95,2.6){$k_0 + u_1$}
 \rput{0}(2,0.2){$c_2$}
 \rput{0}(6.2,0.2){$c_1$}
\end{pspicture}
\caption{The V-SLIP model---In contrast to the bipedal SLIP model, the V-SLIP model has a controllable leg stiffness. This provides two control inputs during the double support phase, but only one control input during the single support phase, rendering the system underactuated.}
\label{fig:vslip_model}
\end{figure}

\subsection{System Dynamics}
\label{subsec:vslip_dynamics}
In deriving the dynamics of the V-SLIP model, we assume that:
\begin{itemize}
\item no slip or bouncing occurs in the foot contact points;
\item the springs are unilateral, meaning that we only allow them to be compressed;
\end{itemize}

The autonomous part of the dynamics of the bipedal V-SLIP model is the same as for the bipedal SLIP model. To include the control inputs, \eqref{eq:slip_dynamics} is extended, arriving at the dynamics for the V-SLIP model in port-Hamiltonian form:
\begin{eqnarray}
\frac{\text{d}}{\text{d} t}\left[ \begin{array}{c} q \\ p \end{array} \right] =
\left[ \begin{array}{cc} 0 & I \\ -I & 0 \end{array} \right]
\left[ \begin{array}{c} \frac{\partial H}{\partial q} \smallskip \\
                        \frac{\partial H}{\partial p} \end{array} \right] +
\left[ \begin{array}{c} 0 \\ B \end{array} \right] u \nonumber \\
y = \left[ \begin{array}{cc} 0 & B^T \end{array} \right]
\left[ \begin{array}{c} \frac{\partial H}{\partial q} \smallskip \\
                        \frac{\partial H}{\partial p} \end{array} \right],
\label{eq:vslip_dynamics}
\end{eqnarray}
with $u = (u_1, u_2)$ the controlled leg stiffness, and $H$ is as defined in Section~\ref{subsec:slip_dynamics}. The input matrix $B$ is given by
\begin{equation*}
B = \left[ \begin{array}{cc} \frac{\partial \phi_1}{\partial q_1} &
                             \frac{\partial \phi_2}{\partial q_1} \smallskip \\
                             \frac{\partial \phi_1}{\partial q_2} &
                             \frac{\partial \phi_2}{\partial q_2} \end{array}
    \right] ,
\end{equation*}
with
%\begin{equation*}
$\phi_1 = -\frac{1}{2} (L_0 - L_1)^2$ and $\phi_2 = -\frac{1}{2} (L_0 - L_2)^2$. %
%\end{equation*}
The output $y$ is dual to $u$, and it is readily verified that the dual product $\langle u | y \rangle$ has the units of power \cite{Duindam2009b}.

As in Section~\ref{subsec:slip_dynamics}, we set $L_2 \equiv L_0$ during the single support phase to eliminate the influence of the swing leg. It is emphasized that the control inputs $u_i,i=1,2$ are restricted, such that the total leg stiffness is physically meaningful, i.e.
\begin{equation}\label{eq:u_limit}
u_i \in \mathbb{R} \mid 0 < k_0 + u_i < \infty.
\end{equation}

\subsection{Controller Design}
\label{subsec:vslip_controller}

The bipedal SLIP model already shows stable walking gaits, with a relatively large basin of attraction \cite{Rummel2010b}. As shown in our previous work, it is possible to tune the spring stiffness $k$ to further increase the robustness of the gait, while minimally modifying the natural dynamics of the walker \cite{Visser2012a}.

The control strategy uses a natural gait of the bipedal SLIP model as reference, i.e. trajectories $(q^\circ(t), \dot{q}^\circ(t))$ that are a solution of \eqref{eq:slip_dynamics}, where $\dot{q}$ is defined as $\dot{q} = M^{-1} p$. However, the bipedal V-SLIP model is underactuated during the single support phase (since there is only one leg in contact with the ground), and thus it is not possible to track $(q^\circ(t), \dot{q}^\circ(t))$ exactly. To avoid that the walker lags behind the reference during the underactuated phase, we propose to instead define a curve in the configuration space by parameterizing $(q^\circ(t), \dot{q}^\circ(t))$ by the horizontal position $q_1$, similar to the approach presented in \cite{Shiriaev2005}. This is possible\footnote{Exact parameterization is not possible, because $q(t)$ is of class $C^2$ only, as outlined in Section~\ref{subsec:slip_dynamics}. Approximating $(q^\circ(t), \dot{q}^\circ(t))$ by finite Fourier series is an alternative that gives satisfactory results, as will be demonstrated.} as long as $\dot{q}_1 > 0$. Then, the desired reference gait can be equivalently described as $(q_2^\ast(q_1),\dot{q}_1^\ast(q_1))$. The control objective is to have the hip trajectory converge to an arbitrary small neighborhood of this reference gait.

In formulating the control strategy, we define the state $x = (q,p)$ and rewrite \eqref{eq:vslip_dynamics} in standard form as
\begin{equation}\label{eq:standard}
\dot{x} = f(x) + \sum_i g_i(x) u_i.
\end{equation}
The following control strategy is proposed.

\begin{proposition}\label{prop:vslip_control}
Given parameterized reference state trajectories $(q_2^\ast,\dot{q}_1^\ast)$, define the error functions
\begin{eqnarray*}
h_1 &= q_2^\ast - q_2, \\
h_2 &= \dot{q}_1^\ast - \dot{q}_1.
\end{eqnarray*}
Then the following control strategy renders solutions of \eqref{eq:vslip_dynamics} asymptotically converging to $(q_2^\ast,\dot{q}_1^\ast)$:
\begin{itemize}
\item during the single support phase, 
 \begin{equation}\label{eq:vslip_uss}
  u_1 = -\frac{1}{\mathcal{L}_{g_1} \mathcal{L}_f h_1} \left(
         \mathcal{L}_f^2 h_1 + \kappa_d \mathcal{L}_f h_1 + \kappa_p h_1 \right)
 \end{equation}
 and $u_2 \equiv 0$;
\item during the double support phase, when the leading leg length satisfies $L_0 - L_e \leq L_1 < L_0$ (i.e. just after the touchdown event), or when the trailing leg length satisfies $L_0 - L_e \leq L_2 < L_0$ (just before the lift-off event), for some small $L_e > 0$:
 \begin{equation}\label{eq:vslip_udse}
 \left[ \begin{array}{c} u_1 \\ u_2 \end{array} \right] = -A^\sharp
 \left( \mathcal{L}_f^2 h_1 + \kappa_d \mathcal{L}_f h_1 + \kappa_p h_1 \right),
 \end{equation}
 with
 \begin{equation*}
 A = \left[ \begin{array}{cc} \mathcal{L}_{g_1} \mathcal{L}_f h_1 & \mathcal{L}_{g_2} \mathcal{L}_f h_1 \end{array} \right],
 \end{equation*}
 and with ${}^\sharp$ denoting the Moore-Penrose\footnote{Since we are addressing a numerical issue, we are not concerned about deriving a proper invariant metric for defining the pseudo-inverse. Instead, we use the Euclidian metric.} pseudo inverse;
\item during the double support phase, when both leg lengths satisfy $L_i < L_0 - L_e$,
 \begin{equation}\label{eq:vslip_uds}
  \left[ \begin{array}{c} u_1 \\ u_2 \end{array} \right] = 
  -A^{-1}
  \left[ \begin{array}{c} \mathcal{L}_f^2 h_1 + \kappa_d \mathcal{L}_f h_1 + \kappa_p h_1 \\
                          \mathcal{L}_f h_2 + \kappa_v h_2 \end{array} \right],
 \end{equation}
 with
 \begin{equation*}
 A = \left[ \begin{array}{cc} \mathcal{L}_{g_1} \mathcal{L}_f h_1 & \mathcal{L}_{g_2} \mathcal{L}_f h_1 \\
                               \mathcal{L}_{g_1} h_2               & \mathcal{L}_{g_2} h_2 \end{array} \right].
 \end{equation*}
\end{itemize}
For any arbitrary small $\varepsilon > 0$, there exist constants $\kappa_p,\kappa_d,\kappa_v > 0$ for the control strategy \eqref{eq:vslip_uss},~\eqref{eq:vslip_udse},~\eqref{eq:vslip_uds} such that:
\begin{equation*} 
\lim_{t \to \infty} q_2^\ast(q_1(t)) - q_2(t) = 0 \;\;\;\; \mbox{and} \;\;\;\; \lim_{t \to \infty} \vert \dot{q}_1^\ast(q_1(t)) - \dot{q}_1(t) \vert < \varepsilon.
\end{equation*}
\end{proposition}

\begin{remark}
The control input \eqref{eq:vslip_udse} is introduced, because the system \eqref{eq:vslip_dynamics} is no longer controllable when one of the legs reaches its rest length $L_0$. As such, the transition domain defined by $L_e$ is necessary to comply with \eqref{eq:u_limit}.
\end{remark}

\begin{remark}
As observed in Section~\ref{subsec:slip_dynamics}, the state trajectories $q(t)$ are of class $C^2$ only, and therefore it is not possible to make the leg stiffness a state of the system, because higher order Lie-derivatives do not exist.
\end{remark}

\begin{remark}
The control inputs, as given by Proposition~\ref{prop:vslip_control}, renders solutions of \eqref{eq:vslip_dynamics} asymptotically converging to $(q_2^\ast,\dot{q}_1^\ast)$, which are parametrized  by the horizontal position $q_1$. Since $q_1$ is strictly monotonically increasing, the control inputs asymptotically stabilize the system \eqref{eq:vslip_dynamics}.
\end{remark}

\begin{proof}
It is straightforward to show that, during the single support phase, $\mathcal{L}_{g_1} \mathcal{L}_f h_1$ in \eqref{eq:vslip_uss} is bounded away from zero if $0 < L_1 < L_0$. Similarly, during the double support phase, the matrix $A$ in \eqref{eq:vslip_uds} is invertible if $0 < L_i < L_0, i=1,2$ and in addition $c_1 \neq c_2$. These conditions are met through the definition of the phase transitions \eqref{eq:vslip_td} and \eqref{eq:vslip_lo}.

During the double support phase, \eqref{eq:vslip_uds} renders the system strongly input-output decoupled, i.e. $h_i$ is invariant under $u_j$ for $i \neq j$ \cite{Nijmeijer1990}. Therefore, and by \eqref{eq:vslip_uss},~\eqref{eq:vslip_udse}, during both the single and double support phases the error dynamics $h_1(t)$ are described by
\begin{equation*}
\ddot{h}_1 + \kappa_d \dot{h}_1 + \kappa_p = 0.
\end{equation*}
If $\kappa_p,\kappa_d$ are chosen such that the zeros of the characteristic polynomial are in the open left half-plane, then the dynamics of $h_1$ are asymptotically stable during each phase. The error function $h_1$ depends on the configuration $q$ only, and $q(t)$ is continuous and differentiable across the phase transitions. Therefore,
\begin{equation*} 
\lim_{t \to \infty} q_2^\ast(q_1(t)) - q_2(t) = 0
\end{equation*}
will be achieved.

The dynamics of the error function $h_2$ are, during the double support phase, described by
\begin{equation*}
\dot{h}_2 + \kappa_v h_2 = 0,
\end{equation*}
which has as analytic solution
\begin{equation*}
h_2(t) = e^{-\kappa_v (t - t_\mathrm{ds})} h_2(t_\mathrm{ds}), \quad t \geq t_\mathrm{ds},
\end{equation*}
where $t_\mathrm{ds}$ is the time instant of the last touchdown event. For any $\kappa_v > 0$, $h_2(t)$ is asymptotically stable during the double support phase.

However, during the single support phase, the dynamics of $h_2(t)$ are uncontrolled. During this phase, the control action of $u_1$ will result in a change of kinetic energy with respect to the constant energy level of the SLIP reference gait. Since $u_1$ is bounded, as defined in \eqref{eq:u_limit}, the total increase in kinetic energy is also bounded. Let $\Delta E_\mathrm{ss}$ denote the increase of energy during the single support phase due to $u_1$. There exists a constant $C_1$ such that
\begin{equation*}
| \Delta E | < C_1,
\end{equation*}
which implies, since $h_2(t)$ is a function of the momentum variable $p_2$, that
\begin{equation*}
| h_2 (t_\mathrm{ds}) | < C_2 < C_1.
\end{equation*}
This in turn implies that there exists $\kappa_v < \infty$ that brings $h_2(t)$ in a neighborhood $\varepsilon$ of zero within the duration of the double support phase.
\end{proof}

With the parameters as listed in Table~\ref{tab:vslip_parameters}, a numeric simulation has been carried out using the PyDSTool software package~\cite{Clewley2007}. The reference $(q_2^\ast,\dot{q}_1^\ast)$ has been obtained by searching for a limit cycle of the uncontrolled SLIP model with the same parameters. As shown in Figure~\ref{fig:vslip_errors}, the controller indeed achieves the converges as claimed in Proposition~\ref{prop:vslip_control}. The error $h_2$ is never exactly zero, because the solutions to \eqref{eq:slip_dynamics} are not analytical. Therefore the parameterized reference is not an exact representation of the natural dynamics, yielding the mismatch in $h_2$ during the single support phase.

\begin{table}
\centering
\caption{Controlled V-SLIP model parameter values.}
\label{tab:vslip_parameters}
\begin{tabular}{lccl}
Parameter       & Value    & Unit       & Description \\
\hline
\hline
$m_h$           & $15.0$   & kg         & Hip mass \\
$L_0$           & $1.0$    & m          & Spring rest length \\
$L_e$           & $0.01$   & m          & Phase transition margin \\
$\alpha_0$      & $62.5$   & ${}^\circ$ & Angle of attack \\
$k_0$           & $2000$   & N/m        & Nominal leg stiffness \\
$k_\mathrm{min}$  & $0$      & N/m        & Lower bound on leg stiffness \\
$k_\mathrm{max}$  & $10000$  & N/m        & Upper bound on leg stiffness \\
\hline
$\kappa_p$      & $350$    &            & Control parameter \\
$\kappa_d$      & $40$     &            & Control parameter \\
$\kappa_v$      & $15$     &            & Control parameter \\
\hline
\end{tabular}
\end{table}

\begin{figure}
\centering
\includegraphics{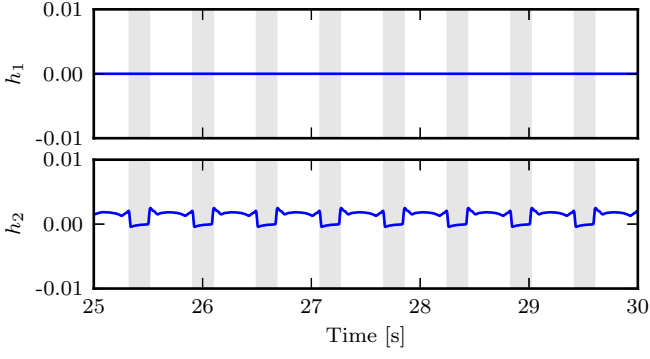}
\caption{Steady-state error functions for the controlled V-SLIP model---It can be seen that for the control parameters listed in Table~\ref{tab:vslip_parameters} the error functions converge like claimed in Proposition~\ref{prop:vslip_control}.}
\label{fig:vslip_errors}
\end{figure}

%%%%%%%%%%%%%%%%%%%%%%%%%%%%%%%%%%%%%%%%%%%%%%%%%%%%%%%%%%%%%%%%%%%%%%%%%%%%%%%%
\section{The Controlled V-SLIP Model with Swing Leg Dynamics}
\label{sec:swing}

While the bipedal (V-)SLIP models do accurately reproduce hip trajectories observed in human walking, and thus can give insights in human walking performance, the models are conceptual. In particular, all mass is assumed to be concentrated in a single point mass at the hip, and the legs are assumed to be massless springs---assumptions that cannot be considered valid for a robotic system.

In this Section we extend the V-SLIP model to incorporate swing leg dynamics. This is done by adding a foot mass, as shown in Figure~\ref{fig:swing_model} \cite{Visser2013a}. During the swing phase, the leg is assumed to have a fixed length $L_0$, while during the stance phase it is again assumed to be a massless spring connecting the foot and the hip masses. In this Section we derive the dynamic equations that govern the system behavior, and extend the controller from Section~\ref{subsec:vslip_controller} to handle the swing leg dynamics.

\begin{figure}
\centering
\begin{pspicture}(7,5.4)
 %\psgrid
 
 \rput[bl]{0}(0.2,0.2){%
  \includegraphics[width=2.625in]{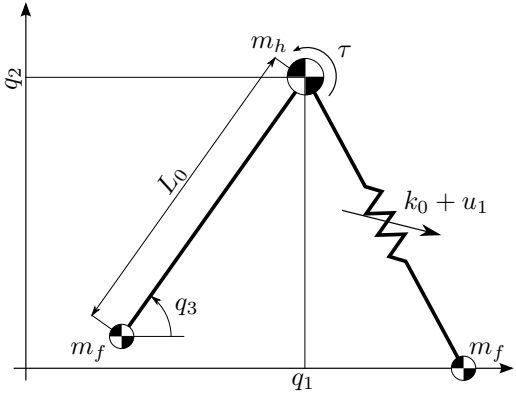}
 }
  
 \rput{0}(4.05,0.2){$q_1$}
 \rput{90}(0.2,4.25){$q_2$}
 \rput{0}(2.5,1.2){$q_3$}
 \rput{0}(3.6,4.7){$m_h$}
 \rput{0}(1.2,0.6){$m_f$}
 \rput{0}(6.5,0.6){$m_f$}
 \rput{0}(5.95,2.6){$k_0 + u_1$}
 \rput{56}(2.3,2.9){$L_0$}
 \rput{0}(4.6,4.6){$\tau$}
\end{pspicture}
\caption{The V-SLIP model with feet---By adding feet masses $m_f$ to the V-SLIP model, swing leg dynamics are introduced. The swing leg is assumed to be constraint at a length $L_0$ during swing, and the stance foot is assumed to be fixed to the ground, i.e. no slip or bouncing in the contact point.}
\label{fig:swing_model}
\end{figure}

\subsection{System Dynamics}
\label{subsec:swing_dynamics}

In deriving the dynamics of the V-SLIP model with feet, we assume that:
\begin{itemize}
\item no slip or bouncing occurs in the foot contact points;
\item the springs are unilateral, meaning that we only allow them to be compressed;
\item during the single support phase, the swing leg is constraint to have length $L_0$.
\end{itemize}
Under these assumptions, during the double support phase, we can use the double support phase model used in Section~\ref{subsec:vslip_dynamics}, and the model behavior is described accordingly by \eqref{eq:vslip_dynamics}.

During the single support phase, the model can be simplified as shown in Figure~\ref{fig:swing_single}. The configuration of the system can be described by $(q_1,q_2,q_3)$, where $q_3 \in [0,\pi)$ is the orientation of the swing leg. The total mass of the swing leg is $m = m_h + m_f$. Since the swing leg is assumed to be a rigid link of length $L_0$, its center of mass is at a distance
\begin{equation*}
d_\mathrm{com} = \frac{m_f L_0}{m_h + m_f}
\end{equation*}
from the hip joint (as indicated in Figure~\ref{fig:swing_single}). The moment of inertia of the leg around its center of mass is given by
\begin{equation*}
J_\mathrm{com} = m_h d_\mathrm{com}^2 + m_f (L_0 - d_\mathrm{com})^2.
\end{equation*}

\begin{figure}
\centering
\begin{pspicture}(7,5.4)
 %\psgrid
 
 \rput[bl]{0}(0.2,0.2){%
  \includegraphics[width=2.625in]{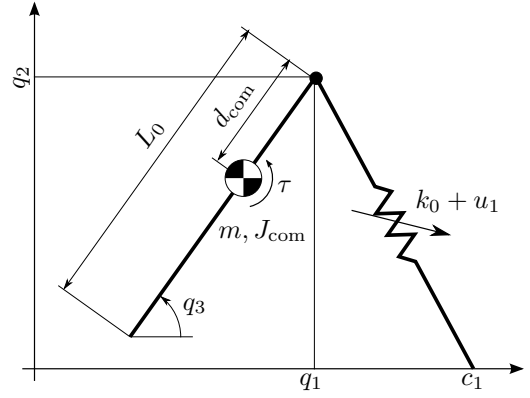}
 }
  
 \rput{0}(4.05,0.2){$q_1$}
 \rput{90}(0.2,4.25){$q_2$}
 \rput{0}(2.5,1.2){$q_3$}
 \rput{0}(3.4,2.2){$m,J_\mathrm{com}$}
 \rput{0}(6.0,2.6){$k_0 + u_1$}
 \rput{56}(1.9,3.5){$L_0$}
 \rput{56}(3.0,3.9){$d_\mathrm{com}$}
 \rput{0}(3.7,2.8){$\tau$}
 \rput{0}(6.2,0.2){$c_1$}
\end{pspicture}
\caption{Model simplification---Under the assumptions of a rigid swing leg and no slip or bouncing in the foot contact point, the model depicted in Figure~\ref{fig:swing_model} can be simplified during the single support phase. During the double support phase, the model is reduced to the V-SLIP model, as shown in Figure~\ref{fig:vslip_model}.}
\label{fig:swing_single}
\end{figure}

In order to derive the dynamic equations of the system for the single support phase, we let $(v_1,v_2,v_3) =: v$ denote the horizontal, vertical and rotational velocity of the (center of mass of the) swing leg. These velocities are related to the rate of change of the configuration variables $\dot{q}$ by the Jacobian matrix $S(q)$, defined as:
\begin{equation}\label{eq:swing_s}
S(q) = \left[ \begin{array}{ccc} 1 & 0 & d_\mathrm{com} \sin (q_3) \\
                                 0 & 1 & -d_\mathrm{com} \cos (q_3) \\
                                 0 & 0 & 1 \end{array} \right],
\end{equation}
such that $v = S(q) \dot{q}$. This allows to have the configuration variables $q$ coincide with those used in the V-SLIP model of Section~\ref{sec:vslip}. In particular, by defining $p := Mv$, with
\begin{equation}\label{eq:swing_mm}
M = \mathrm{diag} (m_h + m_f, m_h + m_f, J_\mathrm{com})
\end{equation}
the mass matrix of the rigid body representing swing leg, the dynamics during the single support phase can be derived in terms of $(q,p)$ as follows.

The kinetic energy is given by $K = \frac{1}{2}p^T M^{-1} p$, and we derive the potential energy function $V$ as
\begin{equation*}
V = (m_h + m_f) g_0 (q_2 - d_\mathrm{com} \sin (q_3)) + \frac{1}{2} k_0 (L_0 - L_1)^2.
\end{equation*}
Then, the Hamiltonian energy function is given by $H = K + V$, and we derive the dynamic equations in port-Hamiltonian form by using the Boltzmann-Hamel equations \cite{Duindam2008}, yielding:
\begin{eqnarray}
\frac{\text{d}}{\text{d} t} \left[ \begin{array}{c} q \\ p \end{array} \right] =
J \left[ \begin{array}{c} \frac{\partial H}{\partial q} \smallskip \\
                          \frac{\partial H}{\partial p} \end{array} \right] +
\left[ \begin{array}{c} 0 \\ B \end{array} \right] u \nonumber \\
y = \left[ \begin{array}{cc} 0 & B^T \end{array} \right]
\left[ \begin{array}{c} \frac{\partial H}{\partial q} \smallskip \\
                        \frac{\partial H}{\partial p} \end{array} \right],
\label{eq:swing_dynamics}
\end{eqnarray}
where the skew-symmetric matrix $J$ is given by
\begin{equation*}
J = \left[ \begin{array}{cc}       0 & S^{-1} \\
                             -S^{-T} & S^{-T} \left(
 \frac{\partial^T (S^T p)}{\partial q} - \frac{\partial (S^T p)}{\partial q}
                                              \right) S^{-1} \end{array} \right].
\end{equation*}
Again, the output $y$ is dual to the input $u$, so that $\langle u | y \rangle$ defines a power flow. The control input $u = (u_1,\tau)$, i.e. the controllable part of the stance leg stiffness, and the torque applied to the swing leg. The input matrix $B$ is given by
\begin{equation*}
B = S^{-T}
 \left[ \begin{array}{cc} \frac{\partial \phi_1}{\partial q_1} & 0 \smallskip \\
                          \frac{\partial \phi_1}{\partial q_2} & 0 \smallskip \\
                          0 & 1 \end{array} \right],
\end{equation*}
with $\phi_1 = -\frac{1}{2} (L_0 - L_1)^2$.
The mapping by $S^{-T}$ is necessary because the inputs are not collocated with $v$, but with $\dot{q}$, as can be seen in Figure~\ref{fig:swing_single}.

\subsection{Phase Transitions}
\label{subsec:swing_phases}

Unlike the V-SLIP model, where in both the double and single support phases the same configuration variables are used, this model uses two sets of configuration variables: in the double support phase only the position of the hip with respect to the foot contact points is relevant, while in the single support phase also the swing leg orientation is required. Therefore, phase transition mappings need to be defined as follows.

\subsubsection*{Transition Conditions}
Similar to the (V-)SLIP models, the touchdown event occurs when the foot of the swing leg has passed in front of the hip,\footnote{Essentially the swing leg is allowed to swing through the floor. This will be addressed in the next model iteration in Section~\ref{sec:knee}.} and in addition, recalling that the swing leg is constraint to have length $L_0$ during the swing phase,
\begin{equation*}
q_2 = L_0 \sin (q_3).
\end{equation*}
At the time instance that both of these conditions are met, the new foot contact point $c_2$ is computed as
\begin{equation*}
c_2 = q_1 - L_0 \cos (q_3).
\end{equation*}

The lift-off event occurs when the trailing leg reaches its rest length $L_0$ with non-zero speed, since we do not allow the springs to pull.\footnote{To be accurate, the transition occurs when the foot starts to accelerate away from the floor. However, this is practically equivalent to assuming that the transition occurs at the moment the spring length becomes equal to its rest length and assuming that the leg instantly becomes rigid at the same moment.}

\subsubsection*{Momentum Variable Mapping}
To complete the phase transitions, the momentum variables of the double support phase need to be mapped to the momentum variables for the single support phase and vice versa. This mapping also needs to ensure that the constraints on the foot contact points are maintained. In particular, upon touchdown, the foot of the former swing leg needs to be instantaneously constraint to fulfill the no-slip condition. This can be realized by applying a momentum reset at the instant of touchdown \cite{Duindam2009a}. It was shown in our previous work that, despite the energy loss associated with the impact, energy-efficient locomotion can be realized \cite{Visser2013a}. However, in this work, we will focus on the added benefit of the compliant legs, and thus assume a compliant impact. This implies that, upon impact, the foot mass $m_f$ will instantaneously dissipate its kinetic energy, while the hip mass $m_h$ remains unaffected by the impact due to the compliant leg.

To map the momentum variables between the phases, we need to account for the disappearing and reappearing of the foot mass. For this purpose, we define new coordinates
\begin{equation*}
z_1 = (q_1,q_2,q_3) \quad \mathrm{and} \quad z_2 = (q_1,q_2,c_i),
\end{equation*}
where $c_i$ denotes the contact point that is subject to change due to the transition. During both the touchdown and the lift-off event, the leg length is assumed to be $L_0$, so that we obtain
\begin{equation*}
z_2(z_1) = \left[ \begin{array}{c} q_1 \\ q_2 \\ q_1 - L_0 \cos (q_3) \end{array} \right].
\end{equation*}
We define the Jacobian matrix $Z := \partial z_2 / \partial z_1$ accordingly.

For the transition from single support to double support, using the subscripts ``old'' and ``new'' for post- and pre-transition values, we have:
\begin{equation*}
\dot{z}_{2,\mathrm{new}} = Z \dot{z}_{1,\mathrm{old}},
\end{equation*}
where $\dot{z}_{1,\mathrm{old}}$ is defined by the momentum variables $p_\mathrm{old}$ just before the phase transition:
\begin{equation*}
\dot{z}_{1,\mathrm{old}} = S^{-1}(q) M_\mathrm{ss}^{-1} p_\mathrm{old},
\end{equation*}
with $S(q)$ defined in \eqref{eq:swing_s} and $M_\mathrm{ss}$ the mass matrix defined in \eqref{eq:swing_mm}. Note that the expression for $\dot{c}_2$ is irrelevant in this phase transition, since we assume that the foot is instantly constraint. The post-transition momentum variables for the double support phase $p_\mathrm{new}$ are then given by
\begin{equation*}
p_\mathrm{new} = M_\mathrm{ds} \underbrace{
 \left[ \begin{array}{c} \dot{q}_1 \\ \dot{q}_2 \end{array} \right]
}_{\in \dot{z}_{2,\mathrm{new}}},
\end{equation*}
with $M_\mathrm{ds}$ the mass matrix defined in \eqref{eq:vslip_mm}.

Similarly, for the transition from double support to single support, we have
\begin{equation*}
\dot{z}_{1,\mathrm{new}} = Z^{-1} \dot{z}_{2,\mathrm{old}},
\end{equation*}
where $\dot{z}_{2,\mathrm{old}}$ is defined through the momentum variables $p_\mathrm{old}$ just before the phase transition:
\begin{equation*}
\dot{z}_{2,\mathrm{old}} = M_\mathrm{ds}^{-1} p_\mathrm{old},
\end{equation*}
with $M_\mathrm{ds}$ the mass matrix defined in \eqref{eq:vslip_mm}, and setting $\dot{c}_1 = 0$, since the foot is stationary at the moment of lift-off. The post-transition momentum variables $p_\mathrm{new}$ for the single support phase are then calculated as
\begin{equation*}
p_\mathrm{new} = M_\mathrm{ss} S(q) \dot{z}_{1,\mathrm{new}},
\end{equation*}
with $S(q)$ defined in \eqref{eq:swing_s}.

\subsection{Controller Design}
\label{subsec:swing_controller}

During the double support phase, the model is equivalent to the V-SLIP model, and therefore, during this phase, the control strategy proposed in Proposition~\ref{prop:vslip_control} can be used. For the single support phase, the control strategy has to be extended to regulate the swing leg trajectory $q_3(t)$. For this purpose, we define a reference trajectory $q_3^\ast(t)$ as a polynomial:
\begin{equation*}
q_3^\ast(t) = \sum_{i=0}^5 a_i (t - t_\mathrm{lo})^i, \quad t_\mathrm{lo} \leq t < t_\mathrm{lo} + T_\mathrm{swing},
\end{equation*}
where $T_\mathrm{swing}$ is the desired swing time, e.g. obtained from the nominal single support phase time of the SLIP model reference gait, and $t_\mathrm{lo}$ is the time instant of the last lift-off event. The coefficients $a_i$ are such that $q_3^\ast(t)$ is a minimum-jerk trajectory with boundary conditions\footnote{The velocity $\dot{q}_3(t_\mathrm{lo})$ and the acceleration $\ddot{q}_3(t_\mathrm{lo})$ are not matched by the reference trajectory $q_3^\ast(t)$, because these quantities are in practice very difficult to measure accurately.}
\begin{equation*}
\left( \begin{array}{c} q_3^\ast(t_\mathrm{lo}) \\
                         \dot{q}_3^\ast(t_\mathrm{lo}) \\
                         \ddot{q}_3^\ast(t_\mathrm{lo}) \end{array} \right) =
\left( \begin{array}{c} q_3(t_\mathrm{lo}) \\
                        0 \\
                        0 \end{array} \right) \\
\end{equation*}
and
\begin{equation*}
\left( \begin{array}{c} q_3^\ast(t_\mathrm{lo} + T_\mathrm{swing}) \\
                        \dot{q}_3^\ast(t_\mathrm{lo} + T_\mathrm{swing}) \\
                        \ddot{q}_3^\ast(t_\mathrm{lo} + T_\mathrm{swing}) \end{array} \right) =
\left( \begin{array}{c} \pi - \alpha_0 \\
                        0 \\
                        0 \end{array} \right).
\end{equation*}

In formulating the control strategy, we will define for both the double support and single support phases a state vector of the form $x = (q,p)$ and write the respective differential equations \eqref{eq:vslip_dynamics} and \eqref{eq:swing_dynamics} in the standard form \eqref{eq:standard}. The following control strategy is proposed, extending the V-SLIP control strategy formulated in Proposition~\ref{prop:vslip_control}.

\begin{proposition}\label{prop:swing_control}
Given reference state trajectories $(q_2^\ast,\dot{q}_1^\ast,q_3^\ast)$, define the error functions
\begin{eqnarray*}
h_1 &= q_2^\ast - q_2, \\
h_2 &= \dot{q}_1^\ast - \dot{q}_1, \\
h_3 &= q_3^\ast - q_3.
\end{eqnarray*}
During the double support phase, the corresponding control strategy formulated in Proposition~\ref{prop:vslip_control} renders solutions of \eqref{eq:vslip_dynamics} asymptotically converging to $(q_2^\ast,\dot{q}_1^\ast)$.

During the single support phase, the following control strategy renders solutions of \eqref{eq:swing_dynamics} asymptotically converging to $(q_2^\ast,q_3^\ast)$:
\begin{equation}\label{eq:swing_uss}
\left[ \begin{array}{c} u_1 \\ \tau \end{array} \right] = -A^{-1}
\left[ \begin{array}{c} \mathcal{L}_f^2 h_1 + \kappa_d \mathcal{L}_f h_1 + \kappa_p h_1 \\
                        \mathcal{L}_f^2 h_3 + \kappa_w \mathcal{L}_f h_3 + \kappa_a h_3 \end{array} \right],
\end{equation}
with
\begin{equation*}
A = \left[ \begin{array}{cc} \mathcal{L}_{g_1} \mathcal{L}_f h_1 & \mathcal{L}_{g_2} \mathcal{L}_f h_1 \\
                            \mathcal{L}_{g_1} \mathcal{L}_f h_3 & \mathcal{L}_{g_2} \mathcal{L}_f h_3 \end{array} \right].
\end{equation*}
For any arbitrary small $\varepsilon_1, \varepsilon_2, \delta > 0$, there exist constants $\kappa_p,\kappa_d,\kappa_v,\kappa_a,\kappa_w > 0$ for the control strategy \eqref{eq:vslip_udse},~\eqref{eq:vslip_uds},~\eqref{eq:swing_uss} such that \begin{equation*} 
\lim_{t \to \infty} \vert q_2^\ast(t) - q_2(t) \vert < \varepsilon_1, \;\;\;\; \mbox{and} \;\;\;\; \lim_{t \to \infty} \vert \dot{q}_1^\ast(t) - \dot{q}_1(t) \vert < \varepsilon_2,
\end{equation*}
and, during $t_\mathrm{lo} \leq t < t_\mathrm{lo} + T_\mathrm{swing}$, 
\begin{equation*}
|q_3^\ast(t_\mathrm{lo} + T_\mathrm{swing}) - q_3(t_\mathrm{lo} + T_\mathrm{swing})| < \delta.
\end{equation*}
\end{proposition}

\begin{proof}
The control strategy is such that the system is strongly input-output decoupled. Therefore, the dynamics of $h_1(t)$ are given by
\begin{equation*}
\ddot{h}_1 + \kappa_d \dot{h}_1 + \kappa_p h_1 = e_1,
\end{equation*}
where $e_1(t)$ is a disturbance due to the phase transitions. As a result, $q_2(t)$ is continuous, but not differentiable. However, $e_1(t)$ is bounded and impulsive, and therefore there exists constants $\kappa_p,\kappa_d > 0$ such that $h_2(t)$ converges to a neighborhood $\varepsilon_1$ of zero.

Similarly, the dynamics dynamics of $h_2(t)$ are given by
\begin{equation*}
\dot{h}_2 + \kappa_v h_2 = e_2,
\end{equation*}
where $e_2(t)$ is also a disturbance due to the phase transitions. As a result of these disturbances, $\dot{q}_1(t)$ is not continuous. However, since $e_2(t)$ is bounded and impulsive, there exist a $\kappa_v > 0$ such that $h_2(t)$ converges to a neighborhood $\varepsilon_2$ of zero.

During the single support phase, the dynamics $h_3(t)$ are given by
\begin{equation*}
\ddot{h}_3 + \kappa_w \dot{h}_3 + \kappa_a h_3 = 0.
\end{equation*}
For suitably chosen constant $\kappa_a,\kappa_w > 0$, such that the zero are in the open left half-plane, the error function $h_3(t)$ converges to a neighborhood $\delta$ of zero in finite time.
\end{proof}

\begin{figure}
\centering
\includegraphics{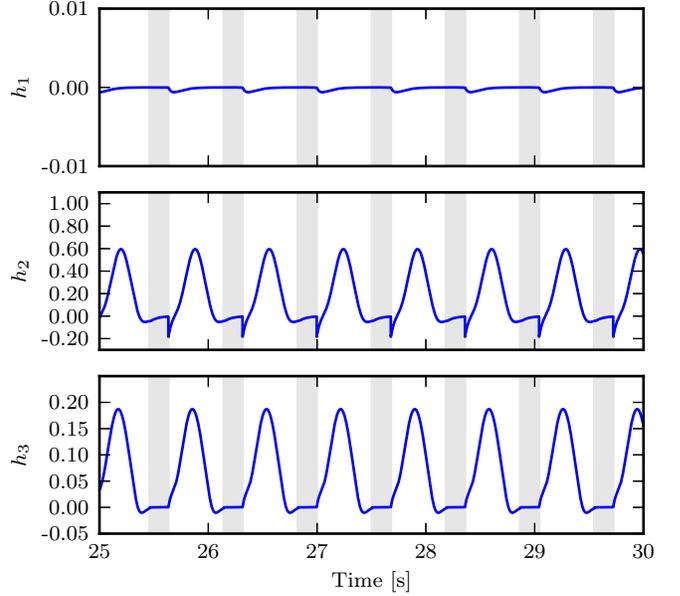}
\caption{Steady-state error functions for the controlled V-SLIP model with swing leg---It can be seen that the error functions converge like claimed in Proposition~\ref{prop:swing_control}. Note that $h_3 \equiv 0$ during the double support phase.}
\label{fig:swing_errors}
\end{figure}

The proposed control strategy has been validated through numeric simulations. The same parameters were used as for the V-SLIP model as listed in Table~\ref{tab:vslip_parameters}, and $m_f = 2.5$~kg. Furthermore, $\kappa_a=1000$ and $\kappa_w=40$. It can be observed that the error functions converge as claimed in Proposition~\ref{prop:swing_control}. In particular, the influence of the swing leg can be clearly observed when the plots are compared to Figure~\ref{fig:vslip_errors}. Specifically, we can see the influence of the swing leg motion in the error function $h_1$ at the onset of the single support phases (the unshaded areas of the plot). The error function $h_2$ also shows a significant increase in amplitude during the swing. We can also observe in $h_2$ the lift-off of the swing leg in the form of discontinuities at the moment of transition from the double support phase (shaded areas) to the single support phase (unshaded areas). The error function $h_3$ shows that the swing leg motion is controlled as claimed by the proposed control law. Note that the degree of freedom $q_3$ is not defined during the double support phase, and therefore $h_3 \equiv 0$ during this phase.

%%%%%%%%%%%%%%%%%%%%%%%%%%%%%%%%%%%%%%%%%%%%%%%%%%%%%%%%%%%%%%%%%%%%%%%%%%%%%%%%
\section{The Controlled V-SLIP model with Retracting Swing Leg Dynamics}
\label{sec:knee}

In this Section, we further refine the model presented in Section~\ref{sec:swing} by adding a knee, as shown in Figure~\ref{fig:knee_model}. This allows the swing leg to be retracted, so that it can be swung forward without scuffing the ground. We derive in this Section the dynamic equations for this model, and further extend the controller.

\begin{figure}
\centering
\begin{pspicture}(7,5.4)
 %\psgrid
 
 \rput[bl]{0}(0.2,0.2){%
  \includegraphics[width=2.625in]{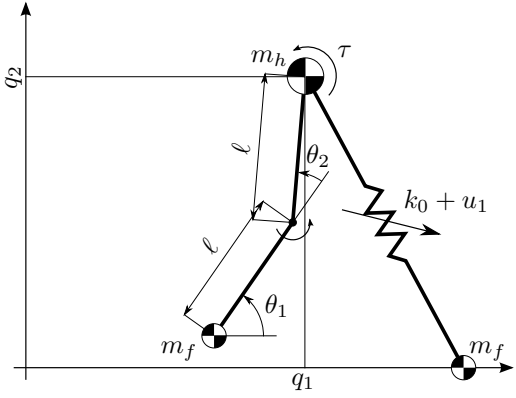}
 }
  
 \rput{0}(4.05,0.2){$q_1$}
 \rput{90}(0.2,4.25){$q_2$}
 \rput{0}(3.7,1.2){$\theta_1$}
 \rput{0}(4.25,3.2){$\theta_2$}
 \rput{0}(3.6,4.5){$m_h$}
 \rput{0}(2.4,0.6){$m_f$}
 \rput{0}(6.5,0.6){$m_f$}
 \rput{0}(5.95,2.6){$k_0 + u_1$}
 \rput{0}(4.6,4.6){$\tau$}
 %\rput{0}(4.3,2.0){$\tau_2$}
 \rput{56}(2.8,2.0){$\ell$}
 \rput{85}(3.22,3.3){$\ell$}
\end{pspicture}
\caption{The V-SLIP model with feet and knees---By adding an actuated knee joint to the model of Section~\ref{sec:swing}, the leg can be retracted during the single support phase. This allows the leg to swing forward without scuffing the ground. It is assumed that no slip or bouncing occurs in the foot contact point of the stance leg.}
\label{fig:knee_model}
\end{figure}

\subsection{System Dynamics}
\label{subsec:knee_dynamics}

Similar to the model presented in Section~\ref{sec:swing}, we will assume that:
\begin{itemize}
\item no slip or bouncing occurs in the foot contact points;
\item the spring are unilateral.
\end{itemize}
These assumptions allow to again use the double support phase model used in Section~\ref{subsec:vslip_dynamics}.

To avoid notational clutter due to goniometric relations, the simplified model depicted in Figure~\ref{fig:knee_single} is used. The simplification is possible, because the introduction of the knee joint introduces only a kinematic relation between the hip mass and the foot mass, since these are located at the extremities of the swing leg.

In deriving the dynamic equations for the single support phase of this model, we define new coordinates as
\begin{equation}\label{eq:knee_z}
z_1 = (q_1,q_2,q_3,q_4) \quad \mathrm{and} \quad z_2 = (q_1,q_2,s_1,s_2)
\end{equation}
where
\begin{eqnarray}
s_1 &= q_1 - q_4 \cos(q_3), \nonumber \\
s_2 &= q_2 - q_4 \sin(q_3),
\label{eq:knee_footpos}
\end{eqnarray}
i.e. the position of the foot of the swing leg. We furthermore define the tangent map $Z = \partial z_2 / \partial z_1$. Using this relation and noting that $z_1 = q$ and thus that $\dot{z}_2 = Z \dot{q}$, we can derive the mass matrix $M(q)$ from the energy equality
\begin{equation}\label{eq:knee_mm}
\frac{1}{2}\dot{q}^T M(q) \dot{q} =
\frac{1}{2} \dot{z}_2^T M_0 \dot{z}_2 = 
\frac{1}{2} \dot{q}^T Z^T M_0 Z \dot{q},
\end{equation}
with $M_0 = \mathrm{diag}(m_h,m_h,m_f,m_f)$.

By defining the momentum variables $p := M(q) \dot{q}$, the kinetic energy $K = \frac{1}{2} p^T M^{-1}(q) p$, and the potential energy function $V$ is found to be:
\begin{equation*}
V = m_h g_0 q_2 + m_f g_0 (q_2 - q_4 \sin (q_3)) + \frac{1}{2} k_0 (L_0 - L_1)^2.
\end{equation*}
Then, the Hamiltonian energy function $H = K + V$ and the dynamic equations in port-Hamiltonian form are given by
\begin{eqnarray}
\frac{\text{d}}{\text{d} t} \left[ \begin{array}{c} q \\ p \end{array} \right] =
\left[ \begin{array}{cc} 0 & I \\ -I & 0 \end{array} \right]
\left[ \begin{array}{cc} \frac{\partial H}{\partial q} \smallskip \\
                         \frac{\partial H}{\partial p} \end{array} \right] +
\left[ \begin{array}{c} 0 \\ B \end{array} \right] u \nonumber \\
y = \left[ \begin{array}{cc} 0 & B^T \end{array} \right]
\left[ \begin{array}{c} \frac{\partial H}{\partial q} \smallskip \\
                \frac{\partial H}{\partial p} \end{array} \right],
\label{eq:knee_dynamics}
\end{eqnarray}
where $u = (u_1,\tau_1,\tau_2)$, i.e. the controllable parts of the stance leg stiffness, and the torques collocated with $q_3$ and $q_4$. The input matrix $B$ is given by
\begin{equation*}
B = \left[ \begin{array}{ccc} \frac{\partial \phi_1}{\partial q_1} & 0 & 0 \smallskip \\
                              \frac{\partial \phi_1}{\partial q_2} & 0 & 0 \smallskip \\
                              0 & 1 & 0 \\
                              0 & 0 & 1 \end{array} \right],
\end{equation*}
with $\phi_1 = -\frac{1}{2} (L_0 - L_1)^2$.

\begin{figure}
\centering
\begin{pspicture}(7,5.4)
 %\psgrid
 
 \rput[bl]{0}(0.2,0.2){%
  \includegraphics[width=2.625in]{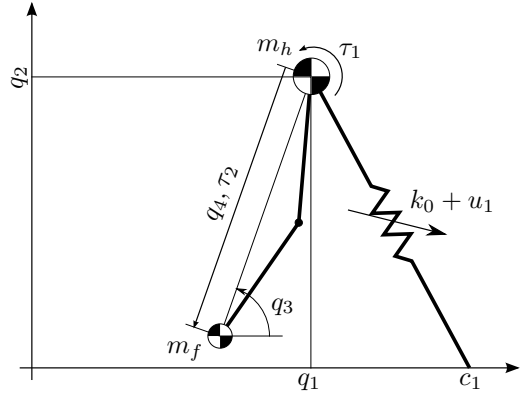}
 }
  
 \rput{0}(4.05,0.2){$q_1$}
 \rput{90}(0.2,4.25){$q_2$}
 \rput{0}(3.7,1.2){$q_3$}
 \rput{71}(2.9,2.7){$q_4,\tau_2$}
 \rput{0}(4.6,4.6){$\tau_1$}
 \rput{0}(3.6,4.65){$m_h$}
 \rput{0}(2.4,0.6){$m_f$}
 \rput{0}(5.95,2.6){$k_0 + u_1$}
 \rput{0}(6.2,0.2){$c_1$}
\end{pspicture}
\caption{Model simplification---The configuration of the swing leg can be equivalently described by a linear degree of freedom $q_4$, corresponding to the distance between the hip and the foot, and the orientation $q_3$, analogous to the model of Figure~\ref{fig:swing_model}. During the double support phase, the model is reduced to the V-SLIP model, as shown in Figure~\ref{fig:vslip_model}.}
\label{fig:knee_single}
\end{figure}

\subsection{Phase Transitions}
\label{subsec:knee_phases}

Just as in the model described in Section~\ref{sec:swing}, also in this model we need to consider the different sets of configuration variables in the single and double support phases. Therefore, in the following the phase transition mappings are defined.

\subsubsection*{Transition Conditions}
The touchdown event occurs when the swing leg foot hits the ground, i.e. when, using \eqref{eq:knee_footpos},
\begin{equation*}
q_2 = q_4 \sin (q_3).
\end{equation*}
At this time instant, the new foot contact point $c_2$ (see Figure~\ref{fig:vslip_model}) is computed as
\begin{equation*}
c_2 = q_1 - q_4 \cos (q_3).
\end{equation*}

The lift-off event is defined the same as in Section~\ref{subsec:swing_phases}, since both models are reduced to the V-SLIP model during the double support phase.

\subsubsection*{Momentum variable mapping}
Similarly to the approach taken in Section~\ref{subsec:swing_phases}, we start from the new set of coordinates defined in \eqref{eq:knee_z} and the corresponding Jacobian matrix $Z$. Thus, for the transition from single support to double support:
\begin{equation*}
\dot{z}_{2,\mathrm{new}} = Z \dot{z}_{1,\mathrm{old}},
\end{equation*}
where $\dot{z}_{1,\mathrm{old}}$ is defined by the pre-transition momentum variables $p_\mathrm{old}$ through
\begin{equation*}
\dot{z}_{1,\mathrm{old}} = M_\mathrm{ss}^{-1} p_\mathrm{old}.
\end{equation*}
Here, $M_\mathrm{ss}$ is the mass matrix defined in \eqref{eq:knee_mm}. As in Section~\ref{subsec:swing_dynamics}, the post-transition momentum variables for the double support phase $p_\mathrm{new}$ are given by
\begin{equation*}
p_\mathrm{new} = M_\mathrm{ds} \underbrace{
 \left[ \begin{array}{c} \dot{q}_1 \\ \dot{q}_2 \end{array} \right]
}_{\in \dot{z}_{2,\mathrm{new}}},
\end{equation*}
with $M_\mathrm{ds}$ the mass matrix defined in \eqref{eq:vslip_mm}.

For the transition from double support to single support, we again have
\begin{equation*}
\dot{z}_{1,\mathrm{new}} = Z^{-1} \dot{z}_{2,\mathrm{old}},
\end{equation*}
where $\dot{z}_{2,\mathrm{old}}$ is defined through the momentum variables $p_\mathrm{old}$ just before the phase transition:
\begin{equation*}
\dot{z}_{2,\mathrm{old}} = M_\mathrm{ds}^{-1} p_\mathrm{old},
\end{equation*}
with $M_\mathrm{ds}$ the mass matrix defined in \eqref{eq:vslip_mm}, and setting $\dot{s}_1 = \dot{s}_2 = 0$, since the foot is stationary at the moment of lift-off. The post-transition momentum variables $p_\mathrm{new}$ for the single support phase are then calculated as
\begin{equation*}
p_\mathrm{new} = M_\mathrm{ss} \dot{z}_{1,\mathrm{new}}.
\end{equation*}

\subsection{Controller Design}
\label{subsec:knee_controller}

During the double support phase, the control strategy proposed in Proposition~\ref{prop:vslip_control} can again be used because of the model correspondence during this phase. For the single support phase, the control strategy has to be extended with respect to the control strategy presented in Proposition~\ref{prop:swing_control}. In particular, in addition to the control of the swing leg orientation, the swing leg length has to be regulated as well. For this purpose, we define a reference trajectory $q_4^\ast(t)$ of the form
\begin{equation*}
q_4^\ast(t) = b_0 + b_1 t + b_2 t^2, \quad t_\mathrm{lo} \leq t \leq t_\mathrm{lo} + T_\mathrm{swing}.
\end{equation*}
The coefficients $b_i$ are such that the trajectory $q_4^\ast(t)$ satisfies the following conditions:
\begin{equation*}
\begin{pmatrix} q_4^\ast(t_\mathrm{lo}) \\
                q_4^\ast(t_\mathrm{lo} + \frac{1}{2} T_\mathrm{swing}) \\
                q_4^\ast(t_\mathrm{lo} + T_\mathrm{swing})
\end{pmatrix} =
\begin{pmatrix} q_4(t_\mathrm{lo}) \\
                L_0 - \Delta \\
                L_0
\end{pmatrix},
\end{equation*}
with $\Delta > 0$ the amount of retraction of the swing leg. This trajectory ensures that at the moment of lift-off the swing leg is immediately accelerating away from the floor, reaching the maximum retraction during the predicted mid-stance. At touchdown, the swing leg will have length $L_0$, corresponding to the (V-)SLIP model.

Defining the state vector of the form $x = (q,p)$ and writing \eqref{eq:vslip_dynamics} and \eqref{eq:knee_dynamics} in the standard form \eqref{eq:standard}, the following control strategy is proposed, extending the strategy formulated in Proposition~\ref{prop:swing_control}.

\begin{proposition}\label{prop:knee_control}
Given reference state trajectories $(q_2^\ast,\dot{q}_1^\ast,q_3^\ast,q_4^\ast)$, define the error functions
\begin{eqnarray*}
h_1 &= q_2^\ast - q_2, \\
h_2 &= \dot{q}_1^\ast - \dot{q}_1, \\
h_3 &= q_3^\ast - q_3, \\
h_4 &= q_4^\ast - q_4.
\end{eqnarray*}
During the double support phase, the corresponding control strategy formulated in Proposition~\ref{prop:vslip_control} renders solutions of \eqref{eq:vslip_dynamics} asymptotically converging to $(q_2^\ast,\dot{q}_1^\ast)$.

During the single support phase, the following control strategy renders solutions of \eqref{eq:knee_dynamics} asymptotically converging to $(q_2^\ast,q_3^\ast,q_4^\ast)$:
\begin{equation}\label{eq:knee_uss}
\left[ \begin{array}{c} u_1 \\ \tau \end{array} \right] = -A^{-1}
\left[ \begin{array}{c} \mathcal{L}_f^2 h_1 + \kappa_d \mathcal{L}_f h_1 + \kappa_p h_1 \\
                        \mathcal{L}_f^2 h_3 + \kappa_w \mathcal{L}_f h_3 + \kappa_a h_3 \\
                        \mathcal{L}_f^2 h_4 + \kappa_n \mathcal{L}_f h_4 + \kappa_\ell h_4
\end{array} \right],
\end{equation}
with
\begin{equation*}
A = \left[ \begin{array}{ccc} \mathcal{L}_{g_1} \mathcal{L}_f h_1 & \mathcal{L}_{g_2} \mathcal{L}_f h_1 & \mathcal{L}_{g_3} \mathcal{L}_f h_1 \\
                    \mathcal{L}_{g_1} \mathcal{L}_f h_3 & \mathcal{L}_{g_2} \mathcal{L}_f h_3  & \mathcal{L}_{g_3} \mathcal{L}_f h_3 \\
                    \mathcal{L}_{g_1} \mathcal{L}_f h_4 & \mathcal{L}_{g_2} \mathcal{L}_f h_4  & \mathcal{L}_{g_3} \mathcal{L}_f h_4
\end{array} \right].
\end{equation*}
For any arbitrary small $\varepsilon_1, \varepsilon_2, \delta_1, \delta_2 > 0$, there exist constants $\kappa_p,\kappa_d,\kappa_v,\kappa_a,\kappa_w,\kappa_\ell,\kappa_n > 0$ for the control strategy \eqref{eq:vslip_udse},~\eqref{eq:vslip_uds},~\eqref{eq:knee_uss} such that 
\begin{equation*} 
\lim_{t \to \infty} \vert q_2^\ast(t) - q_2(t) \vert < \varepsilon_1, \;\;\;\; \mbox{and} \;\;\;\; \lim_{t \to \infty} \vert \dot{q}_1^\ast(t) - \dot{q}_1(t) \vert < \varepsilon_2,
\end{equation*}
%\begin{equation*}
%\lim_{t \to \infty} \vert \dot{q}_1^\ast(t) - \dot{q}_1(t) \vert < \varepsilon_2,
%\end{equation*}
and, during $t_\mathrm{lo} \leq t < t_\mathrm{lo} + T_\mathrm{swing}$,
\begin{eqnarray*}
|q_3^\ast(t_\mathrm{lo} + T_\mathrm{swing}) - q_3(t_\mathrm{lo} + T_\mathrm{swing})| &< \delta_1 \\
|q_4^\ast(t_\mathrm{lo} + T_\mathrm{swing}) - q_4(t_\mathrm{lo} + T_\mathrm{swing})| &< \delta_2
\end{eqnarray*}
\end{proposition}

The proof is analogous to the proof given in Section~\ref{subsec:swing_controller} and is omitted for brevity. The control strategy is validated through numeric simulations, with the same model parameters as used in Section~\ref{subsec:swing_controller}, and in addition $\kappa_\ell=1000$ and $\kappa_n=40$, with a leg retraction $\Delta = 7.5$~cm. The resulting error function plots are shown in Figure~\ref{fig:knee_errors}, and it can be seen that they converge as claimed in Proposition~\ref{prop:knee_control}. The error functions show now significantly bigger influences of the swing leg dynamics when compared to Figure~\ref{fig:swing_errors}.

\begin{figure}
\centering
\includegraphics{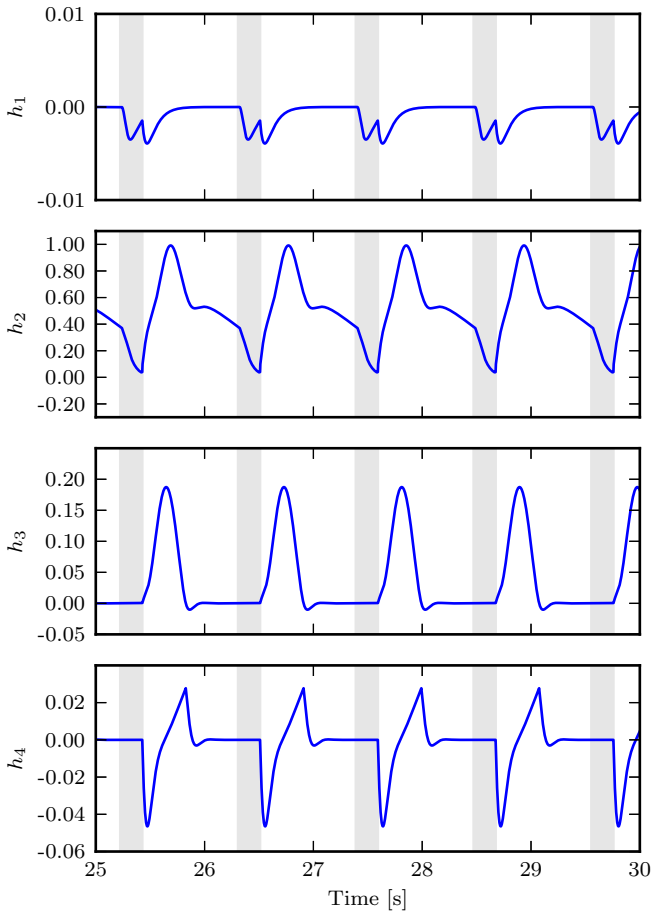}
\caption{Steady-state error functions for the controlled V-SLIP model with leg retraction---It can be seen that the error functions converge like claimed in Proposition~\ref{prop:knee_control}. Note that $h_3 \equiv 0$ and $h_4 \equiv 0$ during the double support phase.}
\label{fig:knee_errors}
\end{figure}

%%%%%%%%%%%%%%%%%%%%%%%%%%%%%%%%%%%%%%%%%%%%%%%%%%%%%%%%%%%%%%%%%%%%%%%%%%%%%%%%
\section{Comparison by Numerical Simulation}
\label{sec:simulation}

Starting from the bipedal SLIP model, three iterations of model refinement have been presented in Section~\ref{sec:vslip}, Section~\ref{sec:swing}, and Section~\ref{sec:knee}. Also, in the first iteration, a robust controller for leg stiffness has been presented in Section~\ref{subsec:vslip_controller}, which has been extended in subsequent iterations. In this section, a comparison of the performance of these controllers is presented.

\subsection{Comparison of Gait Control}

Figure~\ref{fig:hip_traj} shows the horizontal hip trajectory $q_1(t)$, and a detail of the corresponding vertical hip trajectory $q_2(t)$. The most notable difference between the three models is their average forward velocity, which is $1.18$~m/s for the V-SLIP model, but only $1.01$~m/s and $0.64$~m/s for the models including the swing leg dynamics and the knee. Looking at the hip trajectories, it can be seen that this is due to the inclusion of the swing leg dynamics, which introduces a lag in forward motion with respect to the V-SLIP dynamics.

The leg stiffness trajectories $k_i = k_0 + u_i$ are shown in Figure~\ref{fig:control}. Since the reference trajectory is the natural gait of the bipedal SLIP model, it is not surprising that the V-SLIP model requires very little control action to track the reference. The small amount that is required is due to the small mismatch between the parameterized reference trajectory and the true dynamics, as pointed out already in Section~\ref{subsec:vslip_controller}. The introduction of the swing leg dynamics introduces a significant disturbance to the V-SLIP dynamics, exemplified by the larger magnitude of the control inputs. It is interesting to note that the V-SLIP model with swing leg, as introduced in Section~\ref{sec:swing}, requires an impulse-like control input during the single support phase to counter the acceleration and deceleration of the swing leg. In contrast, including the leg retraction (Section~\ref{sec:knee}) results in a smaller moment of inertia, resulting in a smoother control. However, it is noted that the swing leg retraction does result in larger deceleration of the hip, which manifests itself in larger control inputs during the double support phase (bottom plot), actually reaching the lower bound of zero leg stiffness $k_2 = k_0 + u_2$ for short periods of time.

\begin{figure}
\centering
\includegraphics{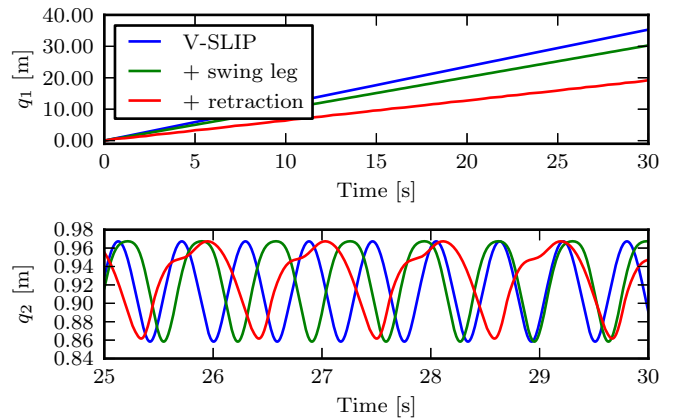}
\caption{Hip trajectories $q_1(t)$ and $q_2(t)$---It can be observed that swinging and retracting the leg has a negative influence on the forward velocity. This influence is particularly apparent in the vertical hip position trajectories.}
\label{fig:hip_traj}
\end{figure}

\begin{figure}
\centering
\includegraphics{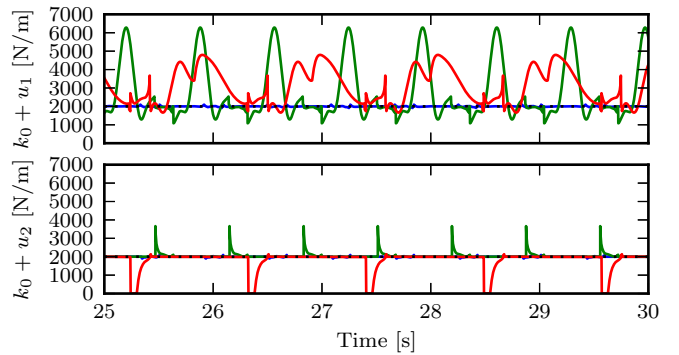}
\caption{Control inputs---See Figure~\ref{fig:hip_traj} for the legend. The controlled V-SLIP model hardly requires any control input to track the reference, which is by design of the control strategy. Adding the swing leg and retracting introduces a significant disturbance in the dynamics, and thus more control input. Note that $u_2 \equiv 0$ during the single support phase.}
\label{fig:control}
\end{figure}

\subsection{Energy Balance}

The natural gait of the bipedal SLIP model is associated to a constant energy level \cite{Geyer2006,Rummel2010b}. Since the V-SLIP model presented in Section~\ref{sec:vslip} matches the bipedal SLIP model, its energy balance is the same if the reference trajectory for the V-SLIP controller exactly matches the solutions of \eqref{eq:slip_dynamics}. However, as already noted before, the solutions of \eqref{eq:slip_dynamics} are not analytical, and therefore a small mismatch between the natural dynamics of the V-SLIP model and the reference is inevitable, resulting in small control action even in nominal conditions, as shown in Figure~\ref{fig:control}.

The energy balance for the model including the swing leg (Section~\ref{sec:swing}) is presented in Figure~\ref{fig:swing_energy}. It can be seen that the introduction of the swing leg dynamics introduces a significant deviation of the constant energy level of the bipedal SLIP model, indicated by the dashed line. The bulges in the kinetic energy plot clearly show the accelerating and the decelerating of the swing leg.

\begin{figure}
\centering
\includegraphics{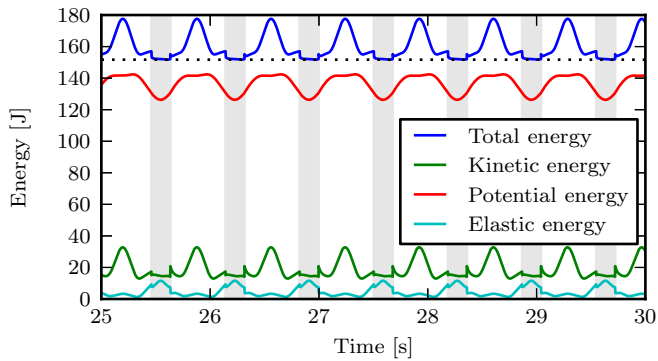}
\caption{Energy balance for the controlled V-SLIP model with swing leg---The dashed line indicates the constant energy level of the bipedal SLIP model. The influence of the swing leg is clearly seen in the bulges in the kinetic energy.}
\label{fig:swing_energy}
\end{figure}

The energy balance for the swing leg model with leg retraction (Section~\ref{sec:knee}) is shown in Figure~\ref{fig:knee_energy}. It clearly shows that the leg retraction further slows down the system, as exemplified by the lower total energy level when compared to Figure~\ref{fig:swing_energy}. However, it is also noted that the bulges observed in Figure~\ref{fig:swing_energy} have been reduced in amplitude in Figure~\ref{fig:knee_energy}. This is because the leg retraction results in a lower swing leg inertia, mitigating the influence of the swing of the leg.

\begin{figure}
\centering
\includegraphics{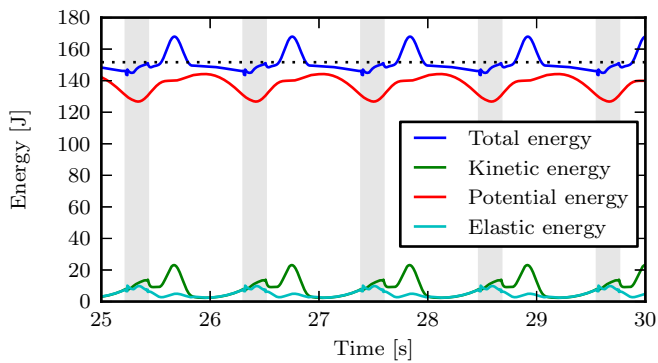}
\caption{Energy balance for the controlled V-SLIP model with leg retraction---The dashed line indicates the constant energy level of the bipedal SLIP model. Also here the influence of swinging and retracting the swing leg is clearly seen.}
\label{fig:knee_energy}
\end{figure}

Both Figure~\ref{fig:swing_energy} and Figure~\ref{fig:knee_energy} show relatively small variations in the total energy level. This signifies that only small amounts of energy are exchanged with the environment and via the control action.

\subsection{Cost of Transport}

Cost of transport (also known as specific resistance) is a measure of energy efficiency, as it measures the energy that a system uses to travel a specified distance \cite{Gregorio1997,Hobbelen2008}. Using the definition proposed in \cite{Hobbelen2008}, the cost of transport is obtained by exploiting the port-Hamiltonian formulation of the dynamic equations \eqref{eq:vslip_dynamics},~\eqref{eq:swing_dynamics},~\eqref{eq:knee_dynamics}:
\begin{equation}\label{eq:cot}
C = \frac{1}{m_\mathrm{total} g_0 \Delta x} \int_T | \langle u | y \rangle | \, \text{d} t, 
\end{equation}
where $m_\mathrm{total}$ denotes the total mass and $\Delta x$ the distance traveled during the time $T$. The cost $C$ captures the amount of energy required for walking the distance $\Delta x$, taking into account that, in general, actuators dissipate energy when negative work is done, rather than storing it.

Using \eqref{eq:cot}, we find $C = 3 \cdot 10^{-3}$ for the controlled V-SLIP model. The cost is not exactly zero due to the aforementioned mismatch between the reference trajectory and the true natural dynamics: if the reference had been exact, we would have $C \equiv 0$. For the model with swing leg (Section~\ref{sec:swing}), we obtain $C = 0.32$. For the model with the retracting swing leg (Section~\ref{sec:knee}), we obtain $C = 0.34$. While Figure~\ref{fig:knee_energy} hinted to lower energy expenditure when compared to Figure~\ref{fig:swing_energy}, this gain in efficiency is offset by the lower average velocity. The cost of transport $C = 0.34$ is in the same range as of human walking \cite{Gregorio1997}, and thus the proposed control strategy allows for a theoretical performance that  s approaches that of human walking.

%%%%%%%%%%%%%%%%%%%%%%%%%%%%%%%%%%%%%%%%%%%%%%%%%%%%%%%%%%%%%%%%%%%%%%%%%%%%%%%%
\section{Conclusions}
\label{sec:conclusions}

In this paper, we started from the bipedal SLIP model, and showed how active leg stiffness variation can render gaits of this model more robust. In particular, it was shown that the model could be extended to include swing leg dynamics, while still embedding the SLIP-like walking behavior by employing feedback control strategies. These control strategies are inspired by the capabilities of the human musculoskeletal capability of varying leg stiffness. It was shown that this approach yields a theoretical cost of transport that is comparable to human performance. This shows that active leg stiffness variation can be an important concept in human walking, which has not yet been transferred to the domain of realizing performant robotic walkers.

The starting point for the analysis presented in this work was the bipedal SLIP model, extended by variable leg stiffness to provide control inputs. Subsequent modeling iterations extended the model to include full swing leg dynamics, while at the same time the control strategy was extended to handle the refined models. The final result is a template model of a bipedal walker, based on the principles of the bipedal SLIP model, and a stabilizing controller that realizes an energetic performance level comparable to human walking. This model plus controller can serve as a basis for control of bipedal robots.

%%%%%%%%%%%%%%%%%%%%%%%%%%%%%%%%%%%%%%%%%%%%%%%%%%%%%%%%%%%%%%%%%%%%%%%%%%%%%%%%
\bibliographystyle{IEEEtran}
\bibliography{bb-clw}

\begin{thebibliography}{10}
\providecommand{\url}[1]{#1}
\csname url@samestyle\endcsname
\providecommand{\newblock}{\relax}
\providecommand{\bibinfo}[2]{#2}
\providecommand{\BIBentrySTDinterwordspacing}{\spaceskip=0pt\relax}
\providecommand{\BIBentryALTinterwordstretchfactor}{4}
\providecommand{\BIBentryALTinterwordspacing}{\spaceskip=\fontdimen2\font plus
\BIBentryALTinterwordstretchfactor\fontdimen3\font minus
  \fontdimen4\font\relax}
\providecommand{\BIBforeignlanguage}[2]{{%
\expandafter\ifx\csname l@#1\endcsname\relax
\typeout{** WARNING: IEEEtran.bst: No hyphenation pattern has been}%
\typeout{** loaded for the language `#1'. Using the pattern for}%
\typeout{** the default language instead.}%
\else
\language=\csname l@#1\endcsname
\fi
#2}}
\providecommand{\BIBdecl}{\relax}
\BIBdecl

\bibitem{BostonDynamics2011}
{Boston Dynamics}, ``{PETMAN} - {BigDog} gets a big brother,'' online:
  http://www.bostondynamics.com/robot\_petman.html, 2011.

\bibitem{Collins2005}
S.~Collins, A.~Ruina, R.~Tedrake, and M.~Wisse, ``Efficient bipedal robots
  based on passive-dynamic walkers,'' \emph{Science}, vol. 307, no. 5712, pp.
  1082--1085, 2005.

\bibitem{Schwab2001}
A.~L. Schwab and M.~Wisse, ``Basin of attraction of the simplest walking
  model,'' in \emph{Proceedings of the ASME Design Engineering Technical
  Conference}, 2001.

\bibitem{Geyer2006}
H.~Geyer, A.~Seyfarth, and R.~Blickhan, ``Compliant leg behaviour explains
  basic dynamics of walking and running,'' \emph{Proceedings of the Royal
  Society B}, vol. 273, no. 1603, pp. 2861--2867, 2006.

\bibitem{Rummel2010b}
J.~Rummel, Y.~Blum, and A.~Seyfarth, ``Robust and efficient walking with
  spring-like legs,'' \emph{Bioinspiration and Biomimetics}, vol.~5, no.~4, p.
  046004, 2010.

\bibitem{Garofalo2012}
G.~Garofalo, C.~Ott, and A.~Albu-Sch{\"a}ffer, ``Walking control of fully
  actuated robots based on the bipedal {SLIP} model,'' in \emph{Proceedings of
  the IEEE International Conference on Robotics and Automation}, 2012, pp.
  1456--1463.

\bibitem{Ketelaar2013}
J.~G. Ketelaar, L.~C. Visser, S.~Stramigioli, and R.~Carloni, ``Controller
  design for a bipedal walking robot using variable stiffness actuators,'' in
  \emph{Proceedings of the IEEE International Conference on Robotics and
  Automation}, 2013, pp. 5650--5655.

\bibitem{vanderLinde1998}
R.~Q. van~der Linde, ``Active leg compliance for passive walking,'' in
  \emph{Proceedings of the IEEE International Conference on Robotics and
  Automation}, 1998, pp. 2339--2344.

\bibitem{Ghorbani2008}
R.~Ghorbani, ``On controllable stiffness bipedal walking,'' Ph.D. dissertation,
  University of Manitoba, 2008.

\bibitem{Jafarian2011}
M.~Jafarian, G.~van Oort, R.~Carloni, and S.~Stramigioli, ``Performance
  comparison of a planar bipedal robot with rigid and compliant legs,'' in
  \emph{Proceedings of the IFAC World Congress}, 2011, pp. 6924--6929.

\bibitem{Vanderborght2013}
B.~Vanderborght, A.~Albu-Schaeffer, A.~Bicchi, E.~Burdet, D.~Caldwell,
  R.~Carloni, M.~Catalano, O.~Eiberger, W.~Friedl, G.~Ganesh, M.~Garabini,
  M.~Grebenstein, G.~Grioli, S.~Haddadin, H.~Hoppner, A.~Jafari, M.~Laffranchi,
  D.~Lefeber, F.~Petit, S.~Stramigioli, N.~Tsagarakis, M.~V. Damme, R.~V. Ham,
  L.~Visser, and S.~Wolf, ``Variable impedance actuators: a review,''
  \emph{Robotics and Autonomous Systems, Elsevier}, vol.~61, pp. 1601--1614,
  2013.

\bibitem{Visser2012a}
L.~C. Visser, S.~Stramigioli, and R.~Carloni, ``Robust bipedal walking with
  variable leg stiffness,'' in \emph{Proceedings of the IEEE International
  Conference on Biomedical Robotics and Biomechatronics}, 2012, pp. 1626--1631.

\bibitem{Duindam2009b}
V.~Duindam, A.~Macchelli, S.~Stramigioli, and H.~Bruyninckx, \emph{Modeling and
  Control of Complex Physical Systems}.\hskip 1em plus 0.5em minus 0.4em\relax
  Springer, 2009.

\bibitem{Shiriaev2005}
A.~Shiriaev, J.~W. Perram, and C.~{Canudas-de-Wit}, ``Constructive tool for
  orbital stabilization of underactuated nonlinear systems: Virtual constraints
  approach,'' \emph{IEEE Transactions on Automatic Control}, vol.~50, no.~8,
  pp. 1164--1176, 2005.

\bibitem{Nijmeijer1990}
H.~Nijmeijer and A.~J. van~der Schaft, \emph{Nonlinear Dynamical Control
  Systems}.\hskip 1em plus 0.5em minus 0.4em\relax Springer, 1990.

\bibitem{Clewley2007}
R.~H. Clewley, W.~E. Sherwood, M.~D. LaMar, and J.~M. Guckenheimer,
  ``{PyDSTool}, a software environment for dynamical systems modeling,''
  online: http://pydstool.sourceforge.net, 2007.

\bibitem{Visser2013a}
L.~C. Visser, S.~Stramigioli, and R.~Carloni, ``Control strategy for energy
  efficient bipedal walking with variable leg stiffness,'' in \emph{Proceedings
  of the IEEE International Conference on Robotics and Automation}, 2013.

\bibitem{Duindam2008}
V.~Duindam and S.~Stramigioli, ``Singularity-free dynamic equations of
  open-chain mechanisms with general holonomic and nonholonomic joints,''
  \emph{IEEE Transactions on Robotics}, vol.~24, no.~3, pp. 517--526, 2008.

\bibitem{Duindam2009a}
------, \emph{Modeling and Control for Efficient Bipedal Walking Robots}.\hskip
  1em plus 0.5em minus 0.4em\relax Springer, 2009.

\bibitem{Gregorio1997}
P.~Gregorio, M.~Ahmadi, and M.~Buehler, ``Design, control, and energetics of an
  electrically actuated legged robot,'' \emph{IEEE Transactions on Systems,
  Man, and Cybernetics}, vol.~27, no.~4, pp. 626--634, 1997.

\bibitem{Hobbelen2008}
D.~G.~E. Hobbelen and M.~Wisse, ``Controlling the walking speed in limit cycle
  walking,'' \emph{International Journal of Robotics Research}, vol.~27, no.~9,
  pp. 989--1005, 2008.

\end{thebibliography}

\end{document}